\documentclass[letterpaper,10pt,conference]{ieeeconf}

\IEEEoverridecommandlockouts
\usepackage{amsmath,amssymb,amsfonts}

\usepackage{amsthm}

\usepackage{algorithm}
\usepackage{algorithmic}
\usepackage{graphicx}
\usepackage{booktabs}
\usepackage{array}
\usepackage{textcomp}
\usepackage{bm}
\usepackage{xcolor}

\usepackage[
    style=ieee,
    sorting=none,
    sortcites=true
]{biblatex}

\usepackage{hyperref}

\definecolor{advisorblue}{RGB}{0,70,170}

\definecolor{checkred}{RGB}{200,0,0}

\newtheorem{theorem}{Theorem}

\theoremstyle{definition}
\newtheorem{definition}{Definition}
\newtheorem{problem}{Problem}

\title{\LARGE \bf
Fast Direction-Conditioned Reachability for Motion Prediction Under Model Uncertainty
}

\author{
Hrishav Das and Melkior Ornik%
\thanks{The authors are with the Department of Aerospace Engineering,
University of Illinois Urbana-Champaign, Urbana, IL 61801, USA.
Emails: \texttt{\{hrishav2,mornik\}@illinois.edu}}%
}

\begin{document}

\maketitle
\thispagestyle{empty}
\pagestyle{empty}

\begin{abstract}
To avoid collisions, a robot must repeatedly predict where nearby agents may
move, usually with an imperfect model of their dynamics. Reachable sets
provide such predictions, but computing them when the system matrices
themselves are uncertain can become computationally expensive and conservative
for frequent replanning. Moreover,
a planner often needs to know only how far an agent can move in one
particular direction, for example toward the robot, rather than the complete
reachable set.

We propose a direction-conditioned reachability method for linear systems
with uncertain state and input matrices. Given a query direction $d$, the
method selects one admissible model $(A^\star,B^\star)$ whose reachable set
extends nearly as far along $d$ as the reachable set of the entire uncertain
model family, and then computes the reachable set of only this model with a
standard reachability solver. On an uncertain linearized bicycle model, the
complete selection-and-computation pipeline is about three times faster than
computing the reachable set of the full uncertain family in the CORA toolbox,
while its extent along $d$ is within $5\%$ of the full family's in the
reported directions. We also use the method in a closed-loop multi-vehicle
simulation in which the robot queries, at each replanning step, how far each
nearby vehicle can move toward it, and replans to avoid the resulting sets.
\end{abstract}

\section{INTRODUCTION}
\label{sec:introduction}


Robots operating around moving agents must repeatedly reason about where those
agents may move before committing to a plan. Reachable sets provide a natural
tool for this purpose: they describe all states a system can attain from an
initial set under admissible inputs and uncertainty. Reachability methods can
broadly be divided into Hamilton--Jacobi methods
\cite{bansal2017hamilton} and set-propagation methods
\cite{annurev:/content/journals/10.1146/annurev-control-071420-081941}.
Hamilton--Jacobi methods can characterize nonlinear reachability but scale
poorly with state dimension \cite{chen2018hamilton,li2021prediction}.
Set-propagation methods scale more favorably to higher-dimensional systems,
but can become conservative, particularly for nonlinear dynamics.

Reachability has become widely used in the robotics literature
\cite{hsu2023safety,wabersich2023data}. In motion planning, reachable sets
are used to restrict candidate trajectories to regions that remain feasible
or safe, and have been incorporated directly into real-time planning
frameworks for robots
\cite{kousik2020bridging,leung2020infusing,kochdumper2024real,seo2022real}.
In autonomous driving, the same idea appears in several forms: reachable sets
have been used to verify safety online \cite{althoff2014online}, to predict the
possible future motion of surrounding traffic participants
\cite{althoff2016set}, and to incorporate those predictions into
collision-aware planning and decision-making
\cite{koschi2020set,li2021prediction}. Related ideas also appear in robust model predictive control. In tube MPC,
bounded model error and disturbances are handled by maintaining a tube around
a nominal trajectory, so that state and input constraints remain satisfied
despite uncertainty \cite{yu2013tube}. Extensions to nonlinear
systems similarly optimize or propagate uncertainty tubes within a receding
horizon framework \cite{lopez2019dynamic}. Tube-based methods provide another example in which uncertainty must be repeatedly accounted for within an online planning or control loop.
\begin{figure}[h]
    \centering
    \includegraphics[width=0.90\linewidth]{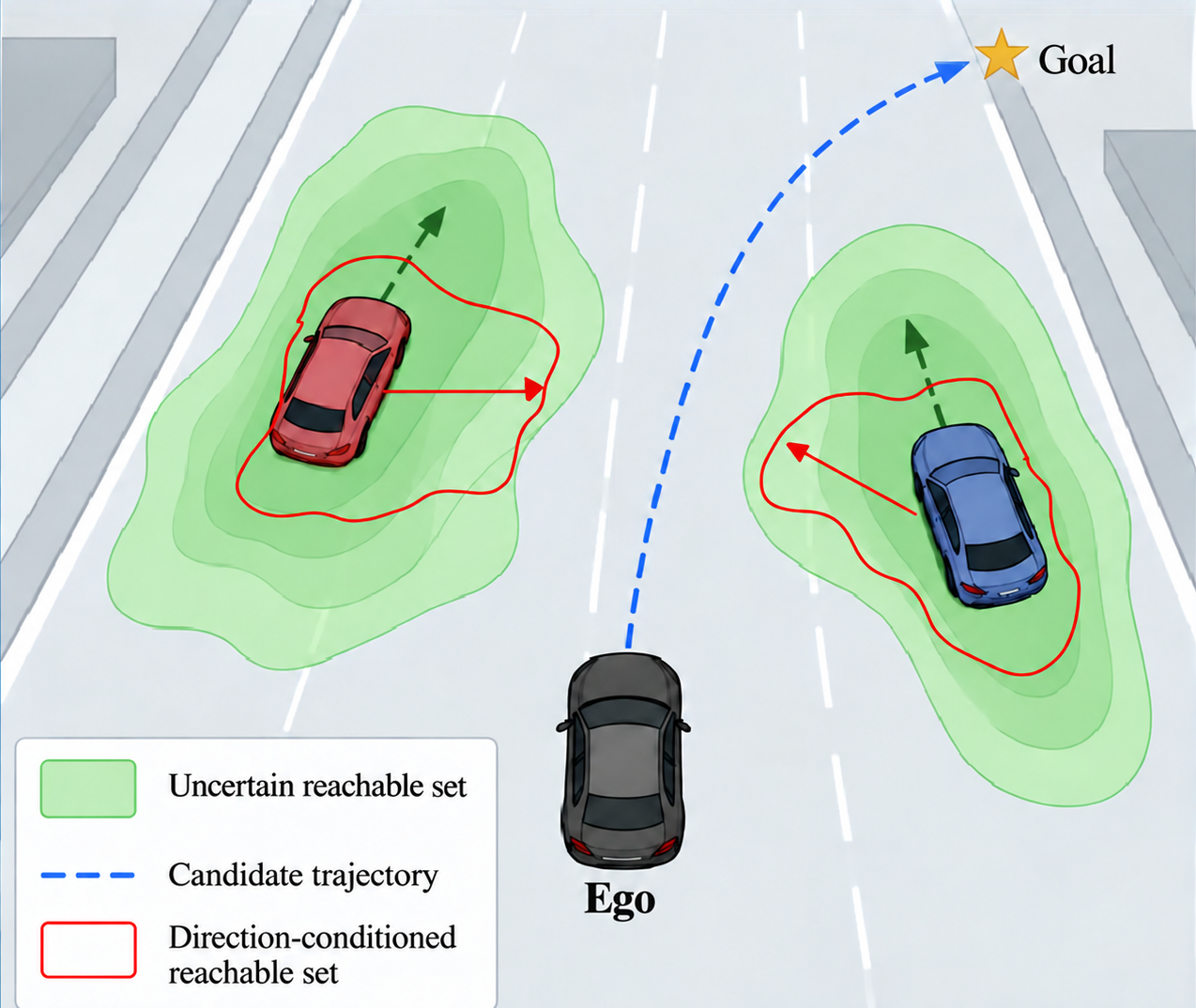}
    \caption{Direction-conditioned reachable sets for planning.}
    \label{fig:motivation_reachability}
\end{figure}

These applications create a need for reachable-set computations that are
both efficient and informative, particularly in the presence of model
uncertainty. When the dynamics themselves are uncertain, a reachability
algorithm must account not only for uncertain states and inputs, but also for
every admissible model in the uncertain family. Prior work reduces the online cost of reachability through efficient set representations \cite{girard2005reachability,le2010reachability,huang2025reachability},
by moving expensive computation offline \cite{kousik2020bridging,majumdar2017funnel},
by restricting the prediction model using interaction-specific information
\cite{bansal2020hamilton,li2021prediction}, or by simplifying the state
representation when only high-level decisions are needed
\cite{kochdumper2024real}.
This computational efficiency matters in online planning, where reachable
sets are recomputed whenever new observations arrive \cite{althoff2014online,koschi2017spot,kochdumper2024real}.

At the same time, many planning decisions do not require the complete
reachable set. Instead, the set is used to answer specific geometric questions, such as whether a candidate
trajectory may collide with another agent
\cite{kousik2020bridging,leung2020infusing,kochdumper2024real}.
For convex sets, many such questions can be phrased using the support
function, which measures how far a set extends along a given direction
\cite{girard2008efficient,le2010reachability}. For collision avoidance, the
relevant quantity may be how far another vehicle can move toward the robot.
If only a few directions matter, computing the full reachable set may
therefore perform unnecessary work (Fig. \ref{fig:motivation_reachability}). Support-function methods represent reachable sets through their
extent along fixed template directions
\cite{girard2008efficient,le2010reachability}.

Building on this observation, we ask a simple but useful question: if the
planner only needs the reachable extent along a task-relevant direction, why
compute the full uncertain reachable set first? We therefore propose a faster
direction-conditioned approach that selects a single admissible model tailored
to the queried direction and propagates only that selected model. The main contributions of this work are:
\begin{itemize}
    \item A direction-conditioned formulation for selecting a state matrix
    $A^\star\in\mathcal A$ and an input matrix $B^\star\in\mathcal B$ from
    sets of uncertain matrices, given a query direction.
    \item A selection-and-computation procedure that, for a single query
    direction, is about three times faster than computing the reachable set
    of the full uncertain family.
    \item A closed-loop multi-vehicle simulation in which these
    predictions are used for replanning around moving vehicles.
\end{itemize}

The remainder of the manuscript is organized as follows.
Section~\ref{sec:problem_formulation} introduces the uncertain system model and
states the direction-conditioned selection problem.
Section~\ref{sec:directional_selection} develops the procedure for selecting
$(A^\star,B^\star)$ and computing the corresponding reachable set.
Section~\ref{sec:bicycle_evaluation} evaluates its directional accuracy and
computation time on an uncertain bicycle model, and
Section~\ref{sec:driving_simulation} applies the method to closed-loop
multi-vehicle collision-avoidance planning.

\section{PROBLEM FORMULATION}
\label{sec:problem_formulation}

Consider the finite-horizon linear system
\begin{equation}
    \dot{x}(t)=Ax(t)+Bu(t),
    \qquad
    x(0)\in\mathcal X_0,
    \qquad
    u(t)\in\mathcal U,
    \label{eq:problem_system}
\end{equation}
for $t\in[0,T]$. The matrices $A$ and $B$ are constant over this horizon,
but their exact values are unknown. We only know that
$A\in\mathcal A$ and $B\in\mathcal B$, where
\begin{equation}
\begin{aligned}
    \mathcal A
    &=
    \left\{
        A_c+\sum_{i=1}^{n_A}\alpha_iG_i^A:
        \alpha_i\in[-1,1]
    \right\},\\
    \mathcal B
    &=
    \left\{
        B_c+\sum_{j=1}^{n_B}\beta_jG_j^B:
        \beta_j\in[-1,1]
    \right\}.
\end{aligned}
\label{eq:matrix_uncertainty_problem}
\end{equation}
Sets of this form are called matrix zonotopes
\cite{girard2005reachability,althoff2011reachable,alanwar2021data}.
Each admissible matrix is a nominal (center) matrix $A_c$ or $B_c$ plus a
weighted sum of generator matrices $G_i^A$ or $G_j^B$. Each generator
describes one source of model error, which may perturb several matrix entries
together, and its coefficient $\alpha_i$ or $\beta_j$ scales that error
between its two extremes. In Sec.~\ref{subsec:prediction_model_uncertainty},
the generators perturb specific terms of a linearized vehicle model.
We assume that $\mathcal X_0$ is compact and convex and that $\mathcal U$ is
a compact convex polytope.

\begin{definition}[Reachable set]
For fixed matrices $(A,B)$, the reachable set at time $T$ is
\begin{equation}
\begin{aligned}
    \mathcal R_{A,B}(T;\mathcal X_0)
    =
    \bigl\{
        x(T)\; \big|\;&
        x(0)\in\mathcal X_0,\\
        &\dot x(t)=Ax(t)+Bu(t),\\
        &u(t)\in\mathcal U
        \text{ for almost every }t\in[0,T]
    \bigr\}.
\end{aligned}
\label{eq:reachable_set_definition}
\end{equation}
\end{definition}

\begin{definition}[Support function]
For a compact set $\mathcal S\subset\mathbb R^n$, its support function in
direction $d\in\mathbb R^n$ is
\begin{equation}
    \rho_{\mathcal S}(d)
    =
    \max_{x\in\mathcal S} d^\top x.
    \label{eq:support_definition}
\end{equation}
\end{definition}
For a unit vector $d$, $\rho_{\mathcal S}(d)$ is how far $\mathcal S$ extends
along $d$ \cite{le2010reachability}. For example, if $d$ points from another
vehicle toward a robot and $\mathcal S$ is that vehicle's reachable set,
$\rho_{\mathcal S}(d)$ measures how far the vehicle can advance toward the
robot.

\begin{problem}[Direction-Conditioned Model Selection]
\label{prob:joint_directional_problem}
Given $\mathcal A,\mathcal B,\mathcal X_0,\mathcal U,T$, and a nonzero
direction $d$, find
\begin{equation}
    \max_{\substack{A\in\mathcal A\\B\in\mathcal B}}
    \rho_{\mathcal R_{A,B}(T;\mathcal X_0)}(d).
    \label{eq:joint_directional_problem}
\end{equation}
\end{problem}

Standard reachability tools for uncertain linear systems compute one set that
encloses the reachable sets of all admissible $(A,B)$ simultaneously
\cite{althoff2011reachable,luo2023reachability}, which requires tracking how
the matrix uncertainty interacts with the state set at every time step.
Instead, we approximately solve Problem~\ref{prob:joint_directional_problem}
to select a single admissible model and pass this fixed model to a standard
reachability solver.

\section{DIRECTION-CONDITIONED REACHABILITY}
\label{sec:directional_selection}

We approximately solve Problem~\ref{prob:joint_directional_problem} in four
steps. First, we write the support function for a fixed pair $(A,B)$.
Second, we select $A^\star$ by maximizing the part of this expression due to
the initial set. Third, with $A^\star$ fixed, we select $B^\star$ by
maximizing the part due to the inputs. Finally, we compute
$\mathcal R_{A^\star,B^\star}(T;\mathcal X_0)$.

\subsection{Directional Support for a Fixed Model}

For fixed $(A,B)$, the solution of \eqref{eq:problem_system} is
\begin{equation}
    x(T)
    =
    e^{AT}x_0
    +
    \int_0^T e^{A(T-t)}Bu(t)\,dt.
    \label{eq:linear_solution}
\end{equation}

Taking the support in direction $d$ gives the standard expression
\cite{le2010reachability}
\begin{equation}
\begin{aligned}
    h_d(A,B)
    &\triangleq
    \rho_{\mathcal R_{A,B}(T;\mathcal X_0)}(d)\\
    &=
    \rho_{\mathcal X_0}\!\left(e^{A^\top T}d\right)
    +
    \int_0^T
    \rho_{\mathcal U}\!\left(
        B^\top e^{A^\top(T-t)}d
    \right)dt .
\end{aligned}
\label{eq:fixed_model_support}
\end{equation}
The first term is the extent along $d$ contributed by the initial set, and
the integral is the extent contributed by the inputs over $[0,T]$.
Problem~\ref{prob:joint_directional_problem} is thus equivalent to
maximizing \eqref{eq:fixed_model_support} over $A\in\mathcal A$ and
$B\in\mathcal B$.

\subsection{State-Matrix Selection}

The matrix $A$ appears in both terms of \eqref{eq:fixed_model_support}.
Rather than optimizing both terms jointly, we select $A$ using only the
first term:
\begin{equation}
    A^\star
    \in
    \arg\max_{A\in\mathcal A}
    \rho_{\mathcal X_0}\!\left(e^{A^\top T}d\right).
    \label{eq:Astar}
\end{equation}
Writing $A(\alpha)=A_c+\sum_{i=1}^{n_A}\alpha_iG_i^A$, this becomes
\begin{equation}
    \max_{\alpha\in[-1,1]^{n_A}}
    \rho_{\mathcal X_0}\!\left(
        e^{A(\alpha)^\top T}d
    \right).
    \label{eq:Astar_coefficients}
\end{equation}
Because $e^{A(\alpha)^\top T}$ depends nonlinearly on $\alpha$, this problem
is nonconvex. In our experiments, we solve it using a dense grid search over
$\alpha$, followed by local refinement.

Selecting $A$ from the first term alone is an approximation: $A^\star$ still
affects the input term, but only through the subsequent $B$-selection step.
We expect the approximation to be most accurate when the input contribution,
rather than the drift dynamics, dominates the reachable extent over the
horizon. We do not
claim that $(A^\star,B^\star)$ solves
Problem~\ref{prob:joint_directional_problem} globally; its accuracy is
evaluated empirically in Sec.~\ref{sec:bicycle_evaluation}.

\subsection{Input-Matrix Selection}

Once $A^\star$ has been selected, define:
\begin{equation}
    p^\star(t)
    =
    e^{A^{\star\top}(T-t)}d.
    \label{eq:pstar}
\end{equation} as in \cite{das2026input}.
Substituting \eqref{eq:pstar} into the second term of
\eqref{eq:fixed_model_support}, the input-matrix selection problem is
\begin{equation}
    B^\star
    \in
    \arg\max_{B\in\mathcal B}
    \int_0^T
    \rho_{\mathcal U}\!\left(B^\top p^\star(t)\right)dt.
    \label{eq:Bstar}
\end{equation}
Let $\mathcal U=\operatorname{conv}\{u_1,\ldots,u_{N_u}\}$. Since a linear
function attains its maximum over a polytope at a vertex,
\begin{equation}
    \rho_{\mathcal U}\!\left(B^\top p^\star(t)\right)
    =
    \max_{i=1,\ldots,N_u}
    p^\star(t)^\top B u_i.
    \label{eq:support_polytope_vertices}
\end{equation}
We denote the objective in \eqref{eq:Bstar} by
\begin{equation}
    J_d(B)
    \triangleq
    \int_0^T
    \max_i
    p^\star(t)^\top B u_i\,dt ,
    \label{eq:Jd_definition}
\end{equation}
and evaluate the integral numerically using midpoint samples on $[0,T]$.

\begin{theorem}[Extreme-point input-matrix selection]
\label{thm:B_extreme_point}
Let
$\mathcal U=\operatorname{conv}\{u_1,\ldots,u_{N_u}\}$
and let $\mathcal B$ be a compact matrix polytope. For fixed $A^\star$,
$J_d(B)$ in \eqref{eq:Jd_definition} is convex in $B$. Consequently, at least
one maximizer of \eqref{eq:Bstar} is an extreme point of $\mathcal B$.
\end{theorem}

\begin{proof}
For each fixed $t$, the integrand is the pointwise maximum of functions that
are linear in $B$, hence convex; integration preserves convexity. A convex
function over a compact polytope attains its maximum at an extreme point
\cite{bertsekas2003convex}.
\end{proof}

The matrix zonotope $\mathcal B$ in \eqref{eq:matrix_uncertainty_problem} is
the affine image of the box $[-1,1]^{n_B}$, so each of its extreme points is
the image of a box vertex
\cite{girard2005reachability,alanwar2021data}. It therefore suffices to
evaluate the at most $2^{n_B}$ candidates
\begin{equation}
    B_\sigma
    =
    B_c+\sum_{j=1}^{n_B}\sigma_jG_j^B,
    \qquad
    \sigma\in\{-1,1\}^{n_B},
    \label{eq:B_vertices}
\end{equation}
and choose
\begin{equation}
    B^\star
    \in
    \arg\max_{\sigma\in\{-1,1\}^{n_B}}
    J_d(B_\sigma).
    \label{eq:Bstar_enumeration}
\end{equation}

\subsection{Reachable Set of the Selected Model}

With both matrices selected, the system
\begin{equation}
    \dot{x}(t)
    =
    A^\star x(t)+B^\star u(t),
    \qquad
    x(0)\in\mathcal X_0,
    \qquad
    u(t)\in\mathcal U
    \label{eq:selected_system}
\end{equation}
has no matrix uncertainty, and we compute
$\mathcal R_{A^\star,B^\star}(T;\mathcal X_0)$ with a standard linear
reachability algorithm \cite{girard2006efficient,le2010reachability}; in
Sec.~\ref{sec:bicycle_evaluation} we use the CORA toolbox
\cite{althoff2015introduction}.

\begin{algorithm}[h]
\caption{Direction-Conditioned Model Selection and Reachability}
\label{alg:directional_reachability}
\begin{algorithmic}[1]
\REQUIRE $\mathcal A,\mathcal B,\mathcal X_0,\mathcal U,d,T$
\STATE Solve \eqref{eq:Astar_coefficients} numerically to obtain $A^\star$
\STATE Compute $p^\star(t)$ from \eqref{eq:pstar}
\STATE Enumerate \eqref{eq:B_vertices} and use \eqref{eq:Bstar_enumeration}
       to obtain $B^\star$
\STATE Compute $\mathcal R_{A^\star,B^\star}(T;\mathcal X_0)$
\RETURN $\mathcal R_{A^\star,B^\star}(T;\mathcal X_0)$
\end{algorithmic}
\end{algorithm}

We call the output of Algorithm~\ref{alg:directional_reachability} the
\emph{direction-conditioned reachable set}, because the model
$(A^\star(d),B^\star(d))$ used to compute it depends on $d$.

\section{BICYCLE MODEL EVALUATION}
\label{sec:bicycle_evaluation}

We compare two reachable-set computations for the same uncertain model: the
uncertain CORA outer approximation $\mathcal R_{\mathrm{unc}}$, which encloses
the reachable sets associated with the matrix family
$\mathcal A\times\mathcal B$, and the direction-conditioned reachable set
$\mathcal R_{A^\star,B^\star}$ of Algorithm~\ref{alg:directional_reachability}.
Both are computed with CORA. We compare their extents along $d$ and the
computation required to obtain them.

\subsection{Vehicle Model and Uncertainty}
\label{subsec:prediction_model_uncertainty}

We use the kinematic bicycle model \cite{kong2015kinematic}
\begin{equation}
    \dot{x} = v\cos\psi,\quad
    \dot{y} = v\sin\psi,\quad
    \dot{\psi} = \frac{v}{L}\tan\delta,\quad
    \dot{v} = a,
\label{eq:bicycle_nonlinear}
\end{equation}
where $(x,y)$ is position, $\psi$ is heading, $v$ is speed, $a$ is
longitudinal acceleration, and $\delta$ is steering angle. We use
$L=0.42\,\mathrm{m}$ as the wheelbase of the simulated small-scale vehicle.

Let
\[
    \xi=[x\;y\;\psi\;v]^\top,
    \qquad
    u=[a\;\delta]^\top .
\]
Around an operating point $(\bar\xi,\bar u)$, the dynamics are locally
approximated by \cite{kong2015kinematic,hebisch2020model}
\[
    \dot{\xi}\approx A_\ell\xi+B_\ell u+c,
\]
where
\begin{equation}
A_\ell=
\begin{bmatrix}
0&0&-\bar v\sin\bar\psi&\cos\bar\psi\\
0&0& \bar v\cos\bar\psi&\sin\bar\psi\\
0&0&0&\dfrac{\tan\bar\delta}{L}\\
0&0&0&0
\end{bmatrix},
\label{eq:bicycle_Al}
\end{equation}
\begin{equation}
B_\ell=
\begin{bmatrix}
0&0\\
0&0\\
0&\dfrac{\bar v}{L\cos^2\bar\delta}\\
1&0
\end{bmatrix},
\label{eq:bicycle_Bl}
\end{equation}
and
\[
    c
    =
    f(\bar\xi,\bar u)
    -A_\ell\bar\xi
    -B_\ell\bar u,
\]
with $f$ denoting the right-hand side of
\eqref{eq:bicycle_nonlinear}. Appending a constant state,
$z=[\xi^\top\;1]^\top$, gives the linear form used in
\eqref{eq:problem_system}:
\begin{equation}
    \dot z=A_cz+B_cu,
    \quad
    A_c=
    \begin{bmatrix}
        A_\ell & c\\
        0_{1\times4} & 0
    \end{bmatrix},
    \quad
    B_c=
    \begin{bmatrix}
        B_\ell\\
        0_{1\times2}
    \end{bmatrix}.
    \label{eq:bicycle_augmented}
\end{equation}

To evaluate the proposed method under model mismatch, we place bounded uncertainty on selected Jacobian and input-gain terms of the local model
\cite{gonultas2023system,seong2024robust}.
We represent these perturbations
using the matrix zonotopes
\begin{equation}
\begin{aligned}
    \mathcal A
    &=
    \left\{
        A_c+\alpha_1G_1^A+\alpha_2G_2^A:
        \alpha_i\in[-1,1]
    \right\},\\
    \mathcal B
    &=
    \left\{
        B_c+\beta_1G_1^B+\beta_2G_2^B:
        \beta_i\in[-1,1]
    \right\},
\end{aligned}
\label{eq:bicycle_matrix_zonotopes}
\end{equation}
where
\begin{equation}
G_1^A
=
0.12
\begin{bmatrix}
0&0&-\bar v\sin\bar\psi&\cos\bar\psi&0\\
0&0& \bar v\cos\bar\psi&\sin\bar\psi&0\\
0&0&0&0&0\\
0&0&0&0&0\\
0&0&0&0&0
\end{bmatrix},
\label{eq:GA1_bicycle}
\end{equation}
\begin{equation}
G_2^A
=
0.15
\begin{bmatrix}
0&0&0&0&0\\
0&0&0&0&0\\
0&0&0&\dfrac{\tan\bar\delta}{L}&0\\
0&0&0&0&0\\
0&0&0&0&0
\end{bmatrix},
\label{eq:GA2_bicycle}
\end{equation}
and
\begin{equation}
G_1^B
=
0.10
\begin{bmatrix}
0&0\\
0&0\\
0&0\\
1&0\\
0&0
\end{bmatrix},
\qquad
G_2^B
=
0.18
\begin{bmatrix}
0&0\\
0&0\\
0&\dfrac{\bar v}{L\cos^2\bar\delta}\\
0&0\\
0&0
\end{bmatrix}.
\label{eq:GB_bicycle}
\end{equation}

Each generator perturbs a related group of entries of the nominal model.
$G_1^A$ changes the sensitivity of the position rates to heading and speed
by up to $\pm12\%$, while $G_2^A$ changes the sensitivity of the yaw rate to
speed by up to $\pm15\%$. Similarly, $G_1^B$ and $G_2^B$ change the
longitudinal-acceleration and steering gains by up to $\pm10\%$ and
$\pm18\%$, respectively.

The $10$--$18\%$ uncertainty levels used here are consistent with errors
reported in vehicle model-identification studies. For example, Seong et al.~\cite{seong2024robust} report average cornering-stiffness estimation errors of approximately $9.5\%$ and $13\%$, while Gonultas et al.~\cite{gonultas2023system} report comparable or larger differences between nominal and identified parameters on a real F1TENTH vehicle. The affine
column $c$ is left unperturbed so that the experiment isolates uncertainty in the Jacobian and input-gain terms.

The initial set $\mathcal X_0$ is an axis-aligned box centered at the
observed state, with half-widths
\[
    [0.07\,\mathrm m,\;
     0.07\,\mathrm m,\;
     0.035\,\mathrm{rad},\;
     0.10\,\mathrm{m/s}]
\]
in $(x,y,\psi,v)$; the appended constant state has zero width. These are
fixed simulation bounds chosen to be of the same order as errors reported for
high-precision vehicle state estimation. In particular,
\cite{morales2019high} reports mean errors of approximately
$4.7\,\mathrm{cm}$ in position, $1^\circ$ in orientation, and
$0.1\,\mathrm{m/s}$ in velocity.

Since the future inputs of the observed vehicle are unknown, the input set
$\mathcal U$ is centered at the current input estimate $\bar u$ with
half-widths
\[
    0.48\,\mathrm{m/s^2}
    \quad\text{and}\quad
    0.105\,\mathrm{rad}
\]
for acceleration and steering, respectively. For the simulations, we
additionally restrict the inputs to
\begin{equation}
    |a|\leq1.4\,\mathrm{m/s^2},
    \qquad
    |\delta|\leq0.42\,\mathrm{rad}.
    \label{eq:bicycle_input_limits}
\end{equation}
The steering bound is consistent with the
$0.4189\,\mathrm{rad}$ steering limit reported for a measured F1TENTH
vehicle in \cite{trumpp2023residual}. The acceleration bound is a
conservative simulation choice rather than a measured actuator limit. The prediction horizon is $T=1.65\,\mathrm{s}$, which corresponds to
approximately four replanning periods in
Sec.~\ref{sec:driving_simulation}.

The results in Secs.~\ref{subsec:accuracy}--\ref{subsec:time} use the
operating point
\begin{equation}
    \bar\xi=
    [0\;\;0\;\;18^\circ\;\;1.35]^\top,
    \qquad
    \bar u=
    [0.10\;\;7^\circ]^\top.
    \label{eq:bicycle_operating_point}
\end{equation}
The nonzero heading and steering angles ensure that both state-matrix
generators are active at the chosen operating point. In particular,
$G_2^A$ vanishes when $\bar\delta=0$, while some entries of $G_1^A$ vanish
when $\bar\psi=0$. The same operating point and uncertainty values are used
for every direction in the evaluation.

\subsection{Directional Accuracy}
\label{subsec:accuracy}

In an application, Algorithm~\ref{alg:directional_reachability} is run for
one or a few directions chosen by the planner, as in
Sec.~\ref{sec:driving_simulation}. To check that its accuracy does not rely
on a favorable choice of $d$, we evaluate it here for many position-space
directions $d(\theta)=[\cos\theta\;\,\sin\theta\;\,0\;\,0\;\,0]^\top$,
$\theta\in[0^\circ,360^\circ)$. The uncertain CORA outer approximation
$\mathcal R_{\mathrm{unc}}$ does not depend on $d$, so it is computed once
and its support is evaluated for every $\theta$. The direction-conditioned
set is recomputed for each $\theta$, since $(A^\star,B^\star)$ changes with
$d$.

For each set $\mathcal R$ and direction $d$, Fig.~\ref{fig:bicycle_directional_extent}
shows the range of the projection $d^\top x$ over $x\in\mathcal R$:
\begin{equation}
    \Bigl[\,\min_{x\in\mathcal R} d^\top x,\;
    \max_{x\in\mathcal R} d^\top x\Bigr]
    =
    \bigl[-\rho_{\mathcal R}(-d),\;\rho_{\mathcal R}(d)\bigr].
    \label{eq:signed_interval}
\end{equation}
Algorithm~\ref{alg:directional_reachability} targets only the upper end
$\rho_{\mathcal R}(d)$. As Fig.~\ref{fig:bicycle_directional_extent} shows,
across $\theta$, the upper extent of $\mathcal R_{A^\star,B^\star}$
closely tracks that of $\mathcal R_{\mathrm{unc}}$, whereas the lower
extents differ
substantially: the selected model is chosen to reach far along $+d$,
not along $-d$.

\begin{figure}[t]
    \centering
    \includegraphics[width=0.9\linewidth]
    {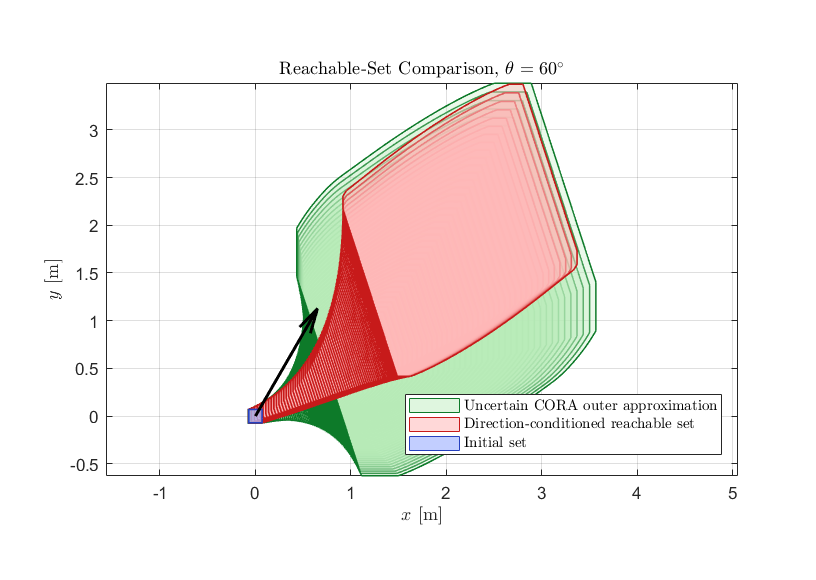}
    \caption{Reachable sets for the query $\theta=60^\circ$. Green: uncertain
    CORA outer approximation $\mathcal R_{\mathrm{unc}}$. Red:
    direction-conditioned reachable set $\mathcal R_{A^\star,B^\star}$ for the
    direction $d$ shown by the arrow. Blue: initial set. Along $d$, the red
    set extends nearly as far as the green set, although it is much smaller in
    other directions.}
    \label{fig:bicycle_60deg_reachset}
\end{figure}

\begin{figure}[t]
    \centering
    \includegraphics[width=0.9\linewidth]{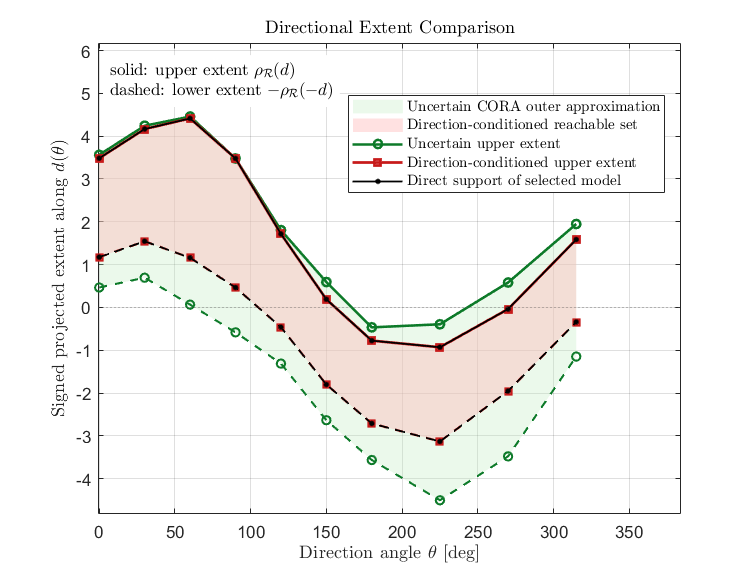}
    \caption{Range \eqref{eq:signed_interval} of $d(\theta)^\top x$ as a
    function of $\theta$. Green: uncertain CORA outer approximation
    $\mathcal R_{\mathrm{unc}}$. Red: direction-conditioned reachable set
    $\mathcal R_{A^\star,B^\star}$. Solid curves show the upper extent
    $\rho_{\mathcal R}(d)$ and dashed curves show the lower extent
    $-\rho_{\mathcal R}(-d)$. Black curves evaluate
    \eqref{eq:fixed_model_support} directly for the selected model.}
    \label{fig:bicycle_directional_extent}
\end{figure}

Table~\ref{tab:bicycle_directional_results} quantifies this for representative
directions. The upper extent of $\mathcal R_{A^\star,B^\star}$ is at most
$4.8\%$ below that of the uncertain CORA outer approximation and matches it at
$\theta=90^\circ$. Because $(A^\star,B^\star)$ is an admissible model and
$\mathcal R_{\mathrm{unc}}$ encloses the reachable sets of all admissible
models, the optimal value of Problem~\ref{prob:joint_directional_problem} lies
between the two columns; the reported gap therefore provides an upper bound
on the loss introduced by the sequential selection in
Sec.~\ref{sec:directional_selection}. Evaluating
\eqref{eq:fixed_model_support} directly for the selected model agrees with
CORA to within $7.5\times10^{-4}$ across all directions.

Fig.~\ref{fig:bicycle_60deg_reachset} shows both sets for $\theta=60^\circ$.
Along $d$, the red direction-conditioned set nearly reaches the green uncertain
CORA boundary, even though $\mathcal R_{A^\star,B^\star}$ is much smaller
overall. This is the intended trade-off: the method focuses on retaining the extent
along the queried direction rather than reproducing the complete uncertain
set.

\begin{table}[t]
    \centering
    \caption{Upper extent along $d(\theta)$ of the uncertain CORA outer
    approximation ($h^+_{\mathrm{unc}}$) and the direction-conditioned set
    ($h^+_{\mathrm{sel}}$), their relative gap, and the time for selecting
    $(A^\star,B^\star)$ ($t_{\mathrm{sel}}$) and for the complete
    Algorithm~\ref{alg:directional_reachability} ($t_{\mathrm{total}}$).
    Computing the uncertain CORA outer approximation takes
    $0.457\,\mathrm s$.}
    \label{tab:bicycle_directional_results}
    \begin{tabular}{c|c|c|c|c|c}
        \hline
        $\theta$ &
        $h_{\mathrm{unc}}^+$ &
        $h_{\mathrm{sel}}^+$ &
        gap [\%] &
        $t_{\mathrm{sel}}$ [s] &
        $t_{\mathrm{total}}$ [s]\\
        \hline
        $0^\circ$   & 3.564 & 3.484 & 2.2 & 0.0378 & 0.153\\
        $30^\circ$  & 4.243 & 4.167 & 1.8 & 0.0375 & 0.167\\
        $60^\circ$  & 4.458 & 4.414 & 1.0 & 0.0373 & 0.146\\
        $90^\circ$  & 3.478 & 3.478 & 0.0 & 0.0356 & 0.158\\
        $120^\circ$ & 1.805 & 1.718 & 4.8 & 0.0339 & 0.155\\
        \hline
    \end{tabular}
\end{table}

\subsection{Computation Time}
\label{subsec:time}

Selecting $(A^\star,B^\star)$ takes $0.0367\,\mathrm{s}$ on average, and
computing the reachable set of the selected model takes $0.118\,\mathrm{s}$,
for a total of $0.154\,\mathrm{s}$ per direction. Computing the uncertain CORA
outer approximation takes $0.457\,\mathrm{s}$, so for a single direction
Algorithm~\ref{alg:directional_reachability} is about three times faster
(Fig.~\ref{fig:bicycle_computation_time}). All timings were measured in
MATLAB R2026a with CORA \cite{althoff2015introduction} on an AMD Ryzen~7
5800H CPU.

This advantage holds only when a few directions are needed. The full uncertain
set can be queried in any number of directions once computed, whereas
Algorithm~\ref{alg:directional_reachability} must be rerun for each
direction. With the timings above, the proposed approach remains faster than
one uncertain propagation for up to two directional queries. This matches the
collision-avoidance setting considered next, where we use one direction per
nearby agent \cite{kochdumper2024real,holmes2020reachable}.

\begin{figure}[t]
    \centering
    \includegraphics[width=0.9\linewidth]{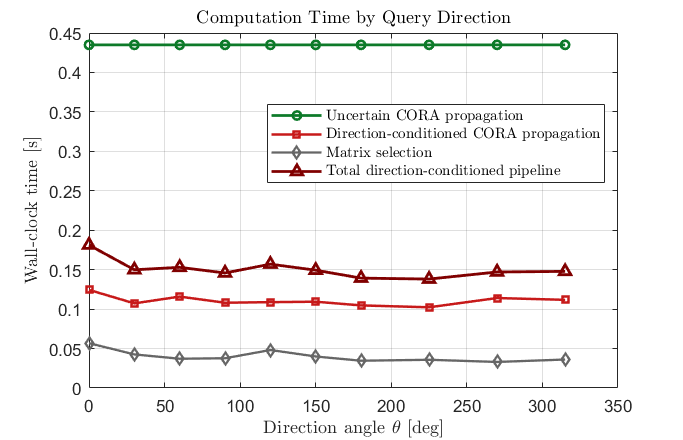}
    \caption{Computation time versus query direction $\theta$. Green:
    uncertain CORA propagation, computed once because it is independent of
    $\theta$. Red: CORA propagation of the direction-conditioned model.
    The remaining curves show matrix-selection time and the total
    direction-conditioned pipeline.}
    \label{fig:bicycle_computation_time}
\end{figure}

\section{CLOSED-LOOP MULTI-VEHICLE NAVIGATION}
\label{sec:driving_simulation}

We now use the method in closed loop. A controlled vehicle, the \emph{ego},
must reach a goal while avoiding two other vehicles whose future motion is
not known exactly. At each replanning step, the ego asks how far each nearby
vehicle can move toward it: the query direction $d$ is the unit vector from
that vehicle to the ego. Only one direction per vehicle is needed, which is
the regime in which Sec.~\ref{subsec:time} shows the method to be faster than
computing the full uncertain reachable set. Set-based prediction of the positions that other traffic participants may
occupy is used similarly in
\cite{koschi2017spot,koschi2020set,li2021prediction}.

\subsection{Scenario and Online Prediction}
\label{subsec:scenario}

The simulation takes place in a $14\,\mathrm{m}\times9\,\mathrm{m}$ planar
workspace. The ego starts at $[0.75\;\,0.85]^\top\,\mathrm m$, must reach
$p_{\mathrm{goal}}=[13.15\;\,8.15]^\top\,\mathrm m$, and moves at a constant
speed of $1.05\,\mathrm{m/s}$. The two other vehicles move according to the
nonlinear bicycle model \eqref{eq:bicycle_nonlinear}, each following its own
sequence of waypoints. The ego does not know these waypoints or the vehicles'
inputs.

The vehicles move in continuous time, but the ego measures their states only
at discrete update times $t_k=k\Delta t$, $k=0,1,\ldots$, where
$\Delta t=0.4\,\mathrm s$ is the replanning period. We describe the prediction for one vehicle and omit the vehicle index; the
same procedure is applied to each vehicle. Let $\xi_k=[x_k\;\,y_k\;\,\psi_k\;\,v_k]^\top$ be the vehicle's
measured state at $t_k$. The ego estimates its acceleration and yaw rate by
finite differences,
\begin{equation}
    \hat a_k
    =
    \frac{v_k-v_{k-1}}{\Delta t},
    \qquad
    \widehat{\dot\psi}_k
    =
    \frac{
        \operatorname{wrap}(\psi_k-\psi_{k-1})
    }{\Delta t},
    \label{eq:observed_input_estimates}
\end{equation}
where $\operatorname{wrap}(\cdot)$ maps the heading difference to
$[-\pi,\pi]$. Inverting the yaw-rate equation
$\dot\psi=(v/L)\tan\delta$ of \eqref{eq:bicycle_nonlinear}
\cite{kong2015kinematic} gives the steering estimate
\begin{equation}
    \hat\delta_k
    =
    \tan^{-1}
    \left(
        \frac{L\,\widehat{\dot\psi}_k}
             {\max(v_k,\varepsilon)}
    \right),
    \label{eq:steering_estimate}
\end{equation}
where $\varepsilon>0$ avoids division by near-zero speeds. Both estimates are
clipped to the actuator limits. We use the measured state and estimated input
as the operating point, $\bar\xi_k=\xi_k$ and
$\bar u_k=[\hat a_k\;\,\hat\delta_k]^\top$, evaluate
\eqref{eq:bicycle_Al}--\eqref{eq:bicycle_Bl} there, and build
$(A_{c,k},B_{c,k})$ and the uncertainty sets
\eqref{eq:bicycle_matrix_zonotopes} as in
Sec.~\ref{subsec:prediction_model_uncertainty}. The prediction model is thus
re-linearized at every update \cite{hebisch2020model,song2019vehicle}.

For vehicle $i$ at position $p_i(t)$ and the ego at $p_{\mathrm{ego}}(t)$,
the query direction is
\begin{equation}
    d_i(t)
    =
    \frac{
        p_{\mathrm{ego}}(t)-p_i(t)
    }{
        \|p_{\mathrm{ego}}(t)-p_i(t)\|_2
    },
    \label{eq:instantaneous_direction}
\end{equation}
padded with zeros to the dimension of $z$ and passed to
Algorithm~\ref{alg:directional_reachability} with $T=1.65\,\mathrm s$.
Direction-conditioned reachable sets are computed only for vehicles within a
$5.3\,\mathrm m$ sensing radius of the ego. The $0.4\,\mathrm s$
replanning period is longer than the time needed to
predict both vehicles with Algorithm~\ref{alg:directional_reachability}
($2\times0.154\approx0.31\,\mathrm s$, Sec.~\ref{subsec:time}), but shorter
than computing both full uncertain reachable sets would take
($2\times0.457\approx0.91\,\mathrm s$).

In the closed-loop figures, green regions show the uncertain CORA outer
approximations $\mathcal R_{\mathrm{unc}}$ associated with the local matrix
uncertainty. Red regions show the direction-conditioned reachable sets
$\mathcal R_{A^\star,B^\star}$. The green sets are shown for comparison,
while the planner uses the red direction-conditioned sets
(Fig.~\ref{fig:directional_motivation}).

\begin{figure}[t]
    \centering
    \includegraphics[width=0.65\linewidth]{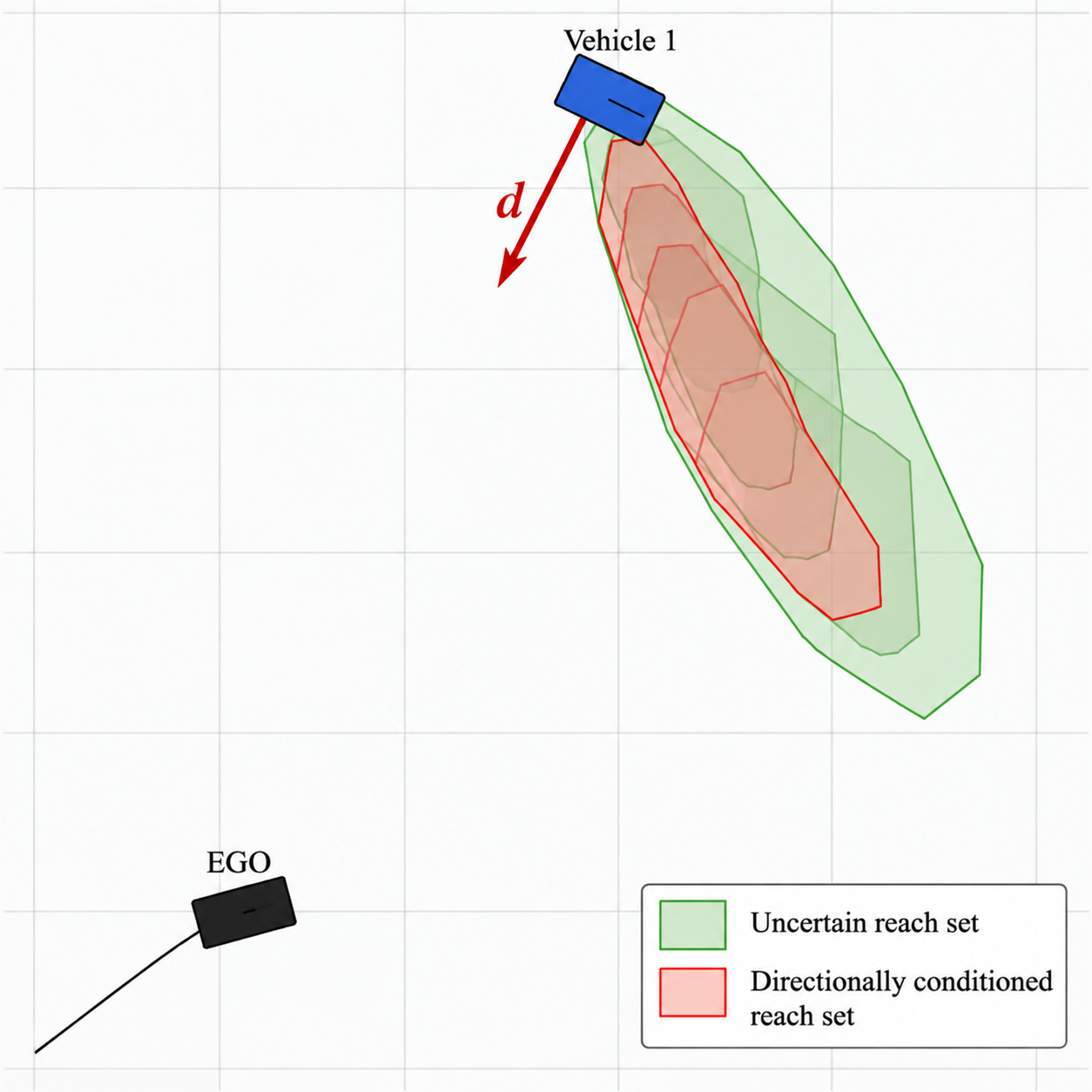}
    \caption{Reachable sets of Vehicle~1 at one update. The arrow $d$ points
    from Vehicle~1 toward the ego. Green: uncertain CORA outer approximation
    $\mathcal R_{\mathrm{unc}}$. Red: direction-conditioned reachable set
    $\mathcal R_{A^\star,B^\star}$, which the planner avoids.}
    \label{fig:directional_motivation}
\end{figure}

\subsection{Reachability-Aware Path Planning}

The ego plans its path with a grid-based A* planner \cite{hart1968formal}.
At each update, for every vehicle within the sensing radius, we compute the
direction-conditioned reachable sets at several times within the horizon,
project them onto the $(x,y)$ plane, and form their union. The result is the
set of positions the vehicle may occupy at some time during the horizon.
We enlarge this set by an additional $0.28\,\mathrm m$ as a safety margin, consistent with set-based occupancy
prediction \cite{koschi2017spot,koschi2020set}, and denote by
$\mathcal H_{\mathrm{dir}}$ the union of these enlarged sets over all nearby
vehicles. The planner then finds the shortest grid path to the goal that
does not enter $\mathcal H_{\mathrm{dir}}$:
\begin{equation}
    \gamma^\star\in\arg\min_{\gamma\in\Gamma} L(\gamma)
    \quad
    \mathrm{s.t.}
    \quad
    \gamma\cap\mathcal H_{\mathrm{dir}}=\emptyset,
    \label{eq:planning_problem}
\end{equation}
where $\Gamma$ is the set of grid paths from the ego's current position to
$p_{\mathrm{goal}}$ and $L(\gamma)$ is the Euclidean length of $\gamma$.
The ego follows $\gamma^\star$ until the next update, when
$\mathcal H_{\mathrm{dir}}$ and $\gamma^\star$ are recomputed.

\subsection{Closed-Loop Behavior}

Fig.~\ref{fig:driving_sequence} shows four frames of the simulation.
Initially, no vehicle is within the sensing radius, and the ego heads
directly toward the goal. (a) When a vehicle enters the sensing radius, its
direction-conditioned reachable set is computed. (b) This set blocks the
ego's planned path, so the planner reroutes around it. (c) A second
interaction later in the run causes another detour. (d) Once no red direction-conditioned set
intersects the route to the goal, the ego proceeds to the goal. The detours
are visible as bends in the ego's executed trajectory.

Instead of the direction-conditioned sets, the planner could use the full
uncertain reachable sets as obstacles. However, based on the timings of
Sec.~\ref{subsec:time}, computing them would not fit within the replanning
period of Sec.~\ref{subsec:scenario}. The direction-conditioned set is not
an enclosure of all possible vehicle motions. Rather, it is selected to
approach the uncertain-family extent toward the ego, the direction most
relevant to the current collision-avoidance decision. This reduces computation
and leaves more free space for the ego in directions that are not currently
queried.

\begin{figure*}[h]
    \centering
    \begin{minipage}[t]{0.48\textwidth}
        \centering
        \includegraphics[width=0.93\linewidth]
        {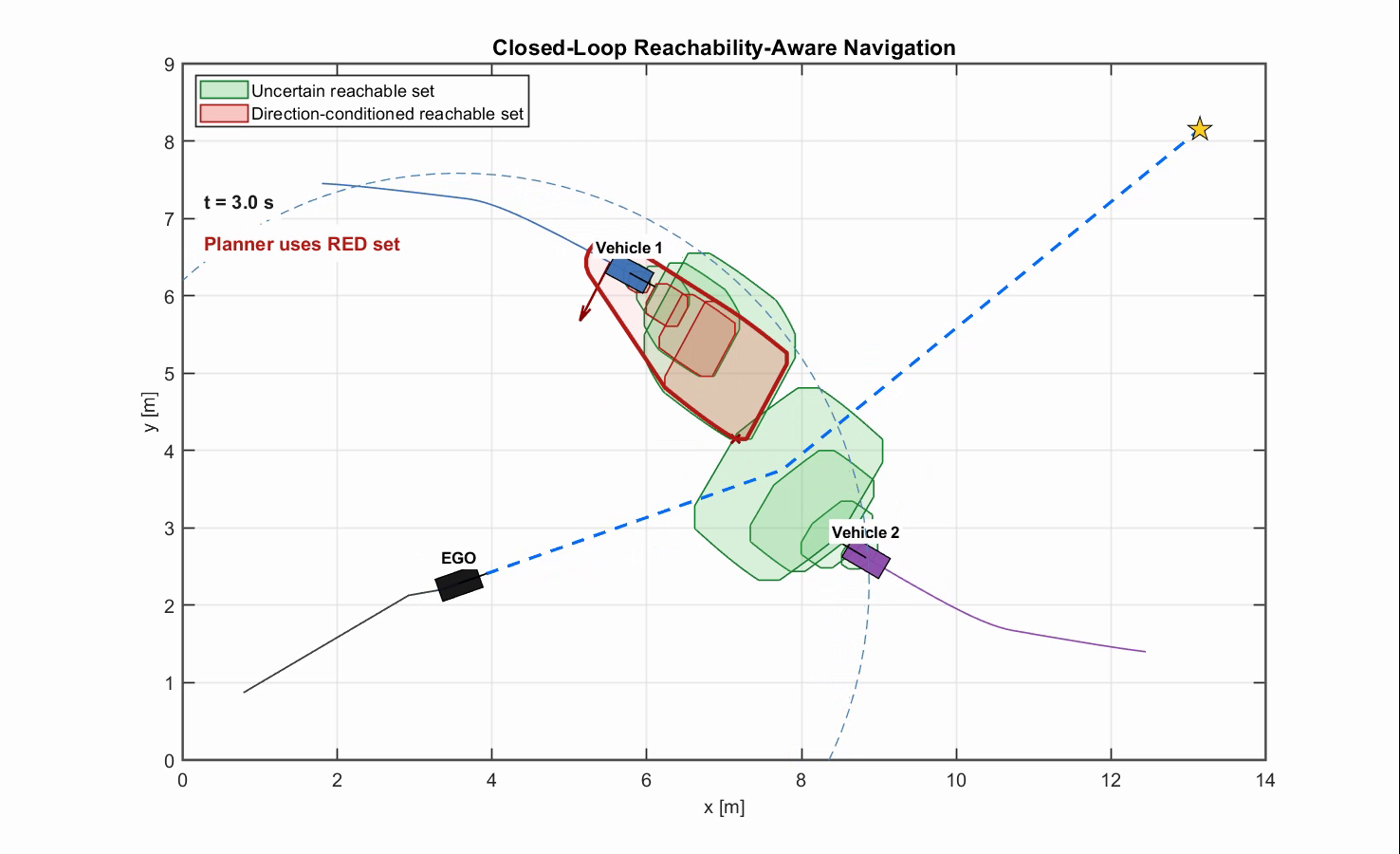}
        \vspace{-1mm}
        {\small (a) A vehicle enters the sensing radius.}
    \end{minipage}
    \hfill
    \begin{minipage}[t]{0.48\textwidth}
        \centering
        \includegraphics[width=0.93\linewidth]
        {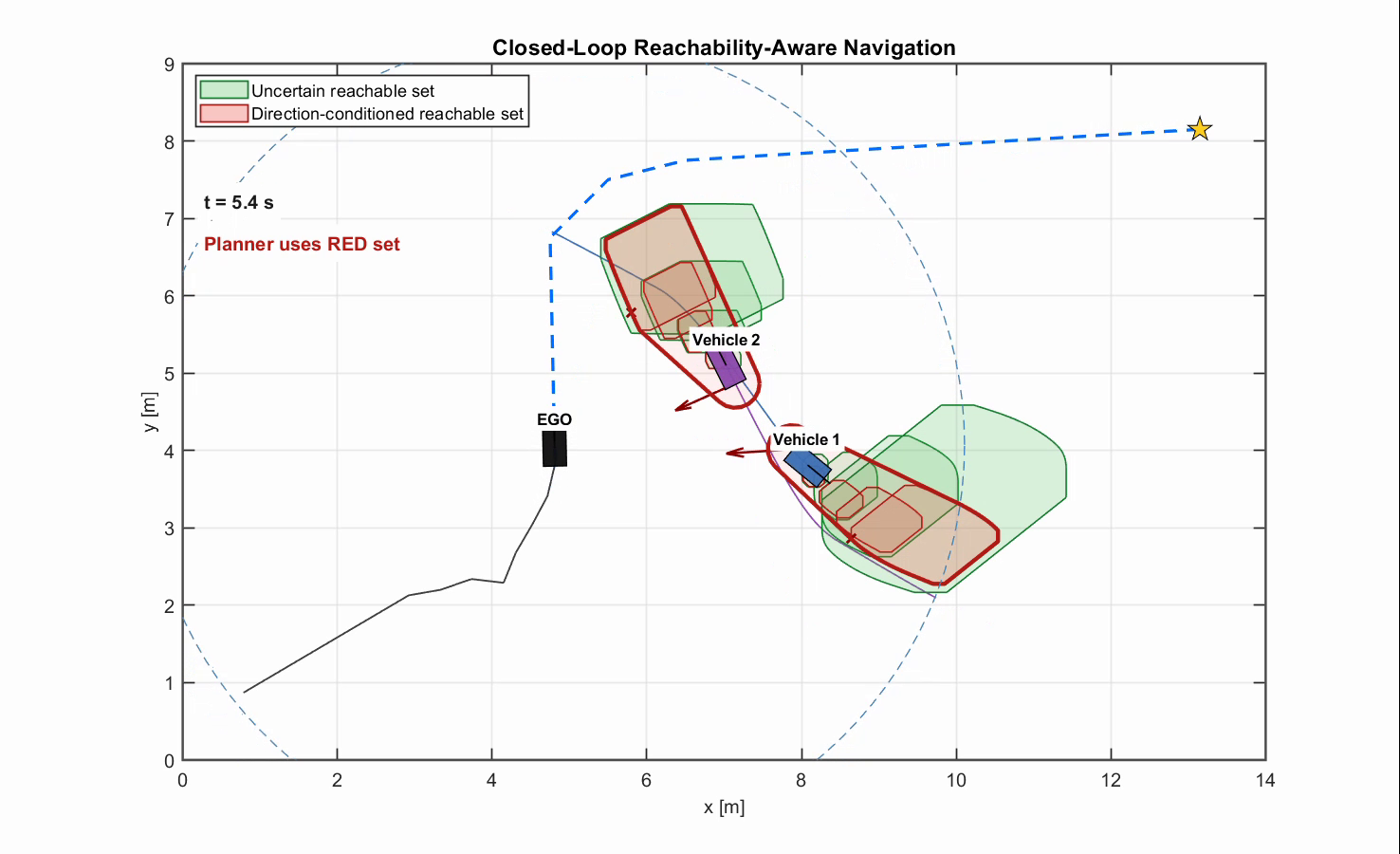}
        \vspace{-1mm}
        {\small (b) First detour.}
    \end{minipage}

    \vspace{2mm}

    \begin{minipage}[t]{0.48\textwidth}
        \centering
        \includegraphics[width=0.93\linewidth]
        {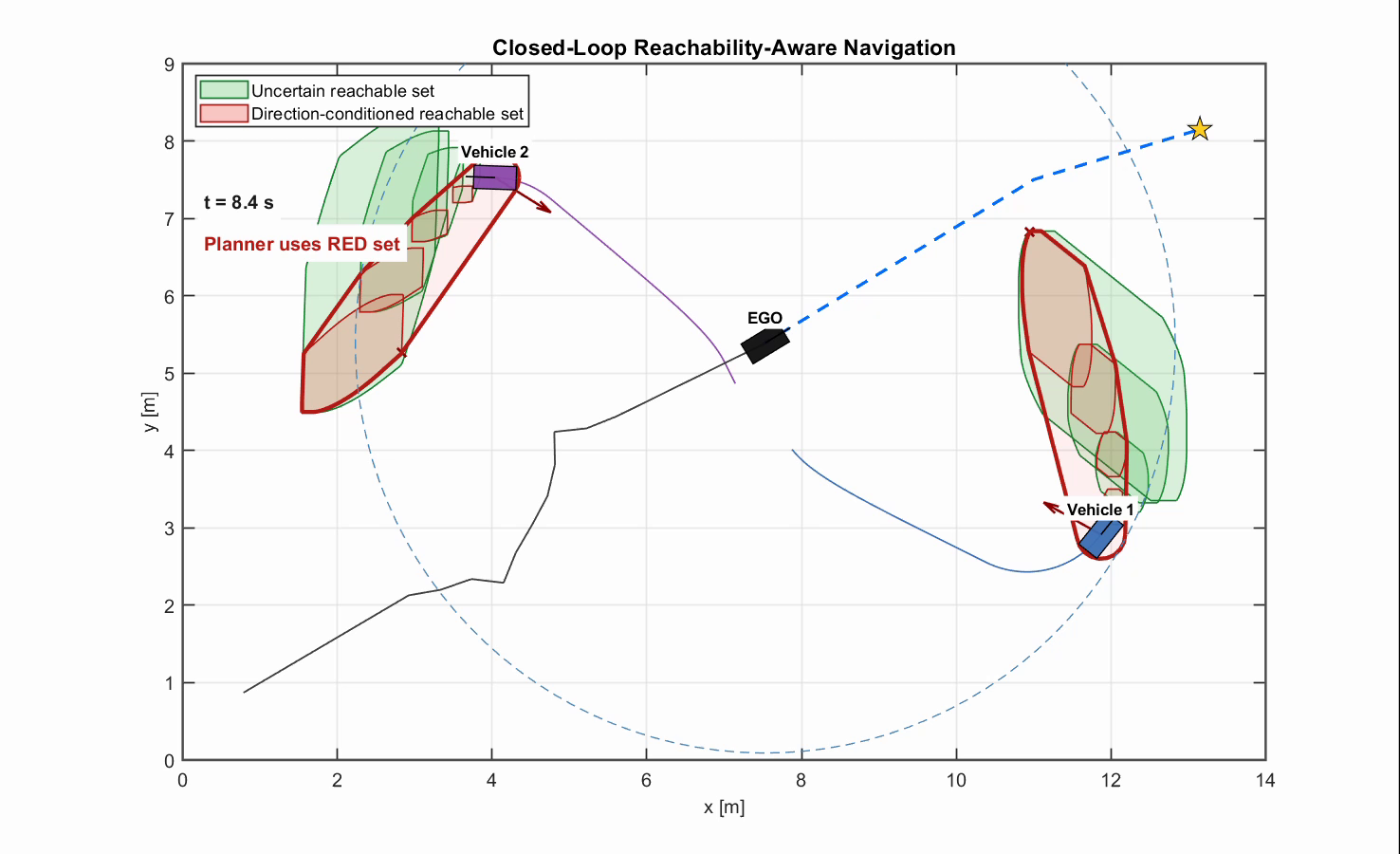}
        \vspace{-1mm}
        {\small (c) Second interaction.}
    \end{minipage}
    \hfill
    \begin{minipage}[t]{0.48\textwidth}
        \centering
        \includegraphics[width=0.93\linewidth]
        {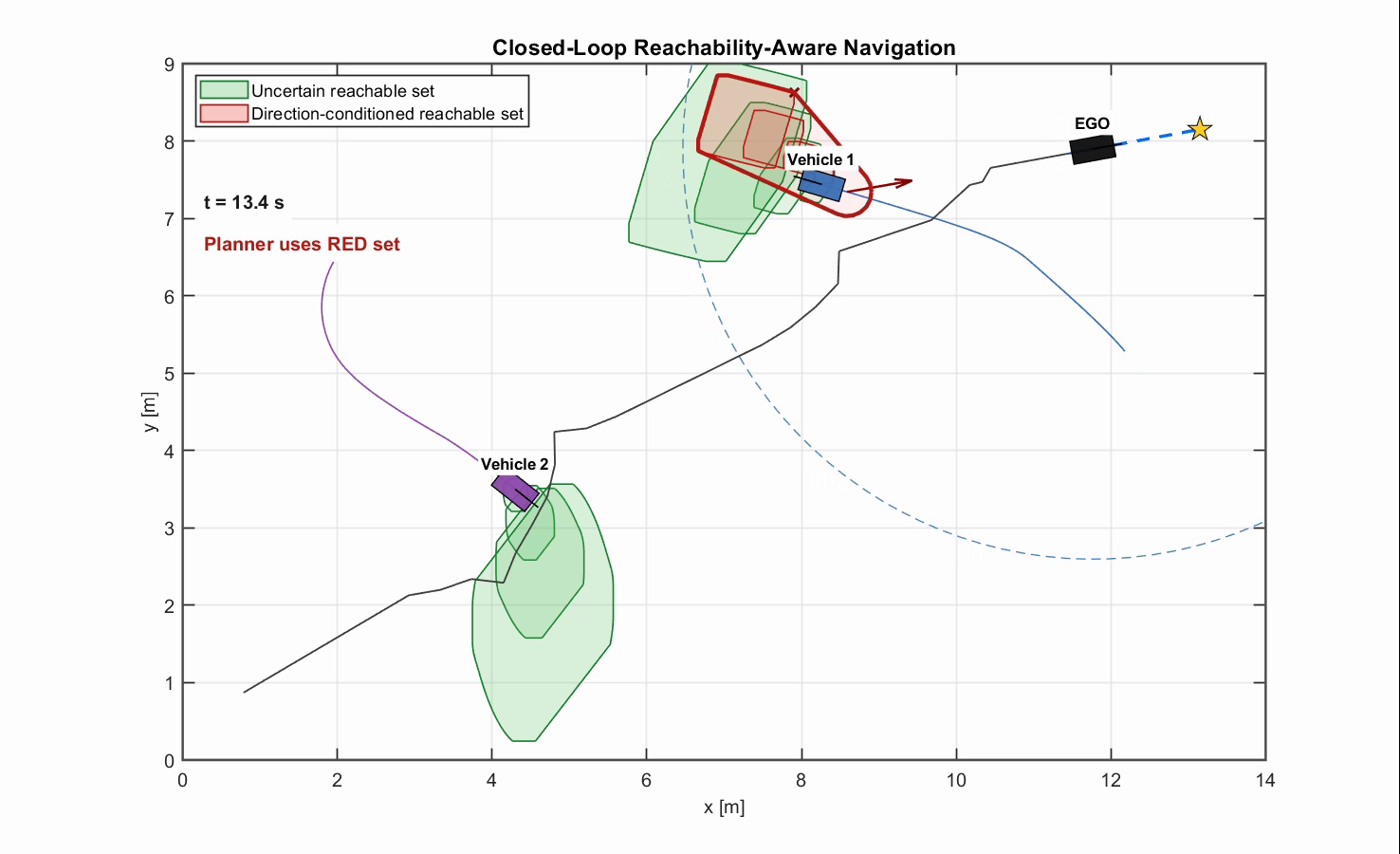}
        \vspace{-1mm}
        {\small (d) Approaching the goal.}
    \end{minipage}

    \vspace{2mm}
    \caption{Closed-loop simulation. Black: ego's executed trajectory.
    Dashed blue: current A* path. Star: goal. Dashed circle: sensing radius.
    Green: uncertain CORA outer approximations
    $\mathcal R_{\mathrm{unc}}$ (reference only). Red:
    direction-conditioned reachable sets $\mathcal R_{A^\star,B^\star}$,
    which the planner avoids. As the vehicles move, the red
    direction-conditioned sets and the A* path are recomputed at every
    update.}
    \label{fig:driving_sequence}
\end{figure*}

\section{CONCLUSION}

We presented a direction-conditioned reachability method for linear systems
whose state and input matrices are uncertain. When a planner needs only the
extent of a reachable set along one direction, the method selects a single
admissible model $(A^\star,B^\star)$ that approximately maximizes this extent
and computes the reachable set of only that model. On an uncertain linearized
bicycle model, the proposed method was about three times faster than computing
the reachable set of the full uncertain model family, while the extent along
the queried direction was within $5\%$ of the uncertain-family extent in the
reported directions.
In a closed-loop multi-vehicle simulation, the resulting predictions were fast
enough for rapid replanning around moving vehicles.

These results indicate that direction-conditioned reachability can serve as a
computational tool for planners that repeatedly query reachable sets along a
few task-relevant directions. Future work will seek stronger guarantees for the state-matrix selection and apply the method to higher-dimensional models and hardware experiments. 

\printbibliography

\end{document}